\documentclass[11pt, dvipsnames, DIV=12]{scrartcl}
\usepackage{silence}
\usepackage[utf8]{inputenc}

\usepackage{ai-involvement}
\AIdeclare[verified]{4}

\usepackage[T1]{fontenc}
\usepackage{amsmath,amssymb,amsthm}
\usepackage{mathtools}
\usepackage{tikz}
\usetikzlibrary{positioning}
\usepackage{microtype}
\usepackage[pdfa,
hidelinks,
pdftex,
pdfdisplaydoctitle,
pdfpagelabels,
pdfauthor={},
pdftitle={The Boolean Power of ReLU},
pdfsubject={},
pdfkeywords={GNNs, MPLang, activation functions},
pdfproducer={Latex with the hyperref package},
pdfcreator={pdflatex}
]{hyperref}
\usepackage{natbib}
\newtheorem{theorem}{Theorem}[section]
\newtheorem{lemma}[theorem]{Lemma}
\newtheorem{fact}[theorem]{Fact}
\newtheorem{proposition}[theorem]{Proposition}
\newtheorem{corollary}[theorem]{Corollary}
\theoremstyle{definition}
\newtheorem{definition}[theorem]{Definition}
\newtheorem{remark}[theorem]{Remark}

\usepackage{fvextra}
\makeatletter
\newcommand{\DiamondPlus}{\mathbin{\mathpalette\DiamondPlus@aux\relax}}
\newcommand{\DiamondPlus@aux}[2]{%
  \ooalign{%
    $\m@th#1\Diamond$\cr
    \hidewidth$\m@th#1\raisebox{0.0ex}{\scalebox{1}[1.1]{$+$}}$\hidewidth\cr
  }%
}
\makeatother
\newcommand{\R}{\mathbb{R}}
\newcommand{\Q}{\mathbb{Q}}
\newcommand{\N}{\mathbb{N}}

\newcommand{\B}{\mathbb{B}}
\newcommand{\relu}{\mathrm{ReLU}}
\newcommand{\trelu}{\mathrm{TrReLU}}
\newcommand{\bool}{\mathrm{bool}}
\newcommand{\dm}{\DiamondPlus}
\newcommand{\Estar}{E_\star}

\newcommand{\sgn}{\mathrm{sgn}}
\newcommand{\MPLang}{\text{-}\mathrm{MPLang}}

\title{The Boolean Power of ReLU}
\author{
Pablo Barcel\'o\\
\texttt{\small pbarcelo@uc.cl}\\
\small Pontifical Catholic University\\
\small IMFD \& CENIA,
 Chile
\and
Floris Geerts\\
\texttt{\small floris.geerts@uantwerp.be}\\
\small University of Antwerp\\
\small Belgium
\and
Matthias Lanzinger\\
\texttt{\small matthias.lanzinger@tuwien.ac.at}\\
\small TU Wien\\
\small Austria
\and
Klara Pakhomenko\\
\texttt{\small klara.pakhomenko@uhasselt.be}\\
\small Universiteit Hasselt\\
\small Belgium
\and
Jan Van den Bussche\\
\texttt{\small jan.vandenbussche@uhasselt.be}\\
\small Universiteit Hasselt\\
\small Belgium
}
 \date{}
\begin{document}
\emergencystretch=2em
\maketitle
\begin{abstract}
We prove that, on finite simple undirected graphs equipped with a single Boolean node feature, the Boolean queries expressible in $\Sigma\MPLang$, for \emph{any} collection $\Sigma$ of eventually constant activation functions and with arbitrary real coefficients, form a strict subclass of the Boolean queries expressible in $\mathrm{ReLU}\MPLang$. We thereby settle a recently posed open problem: whether $\mathrm{ReLU}\MPLang$ is more powerful than $\mathrm{TrReLU}\MPLang$ when it comes to Boolean queries. In
particular, this
implies that $\mathrm{ReLU}$-GNNs are strictly more expressive than
$\{\mathrm{TrReLU},\mathrm{id}\}$-GNNs with respect to Boolean queries on Boolean-featured graphs.
\end{abstract}
\section{Introduction}
Graph neural networks (GNNs) compute node embeddings by repeatedly combining a
node's current features with aggregated features from its neighbours, followed
by a pointwise activation function.  $\Sigma\MPLang$ (for
\emph{$\Sigma$-Message-Passing Language}) provides a small declarative
language for precisely this computation: its expressions are built from node
features, affine combinations, neighbourhood aggregation, and activations from a set $\Sigma$.
This makes $\Sigma\MPLang$ a convenient formalism for isolating how architectural
choices affect GNN expressivity.  In particular, $\Sigma\MPLang$ has the same
expressive power as GNNs whose layers may use either $\sigma\in\Sigma$ or the identity
activation (i.e., a skip layer that passes a node's incoming value through
unchanged), while $\mathrm{ReLU}\MPLang$ captures ReLU-GNNs
\citep{Geerts_2022}.
Recent work showed that the unbounded ReLU is strictly more expressive than the
bounded truncated ReLU for \emph{numerical} queries, but left open whether this
advantage survives Booleanisation, where the numerical output at a node is
thresholded to obtain a Boolean query \citep{Barcelo+2026}.  This does not
follow from numerical separation: two languages may compute different
real-valued embeddings while still defining exactly the same Boolean queries.
Nor is the question settled by the Boolean separation result by Benedikt, Lu and Tan
\cite{benedikt2025decidabilitygraphneuralnetworks}, since that result concerns GNNs without the identity
layers implicit in $\Sigma\MPLang$.  Because such linear layers can themselves increase
expressive power, it remained possible that they eliminate the Boolean gap
between truncated ReLU and ReLU.
We show that they do not.  Already on finite simple undirected graphs equipped
with a single Boolean node feature,
\[
  \bool\bigl(\mathrm{TrReLU}\MPLang\bigr)
  \subsetneq
  \bool\bigl(\mathrm{ReLU}\MPLang\bigr),
\]
where $\bool(\cdot)$ denotes Booleanisation.
More generally, we exhibit a Boolean ReLU query that is not definable in
$\Sigma\MPLang$ for any collection $\Sigma$ of eventually constant activation
functions, even with arbitrary real coefficients.  Thus, the unboundedness of
ReLU affects not only the numerical values of graph embeddings, but also the
node properties that message-passing GNNs can recognise after thresholding.

\begin{aidisclosure}[verified]{4}
Generative AI tools were used during the development and preparation of this manuscript. In particular, ChatGPT 5.6 Sol (Pro) was used to explore proof strategies and generated the proof of the final result from a single prompt. Claude Fable 5 and the authors subsequently checked, revised where necessary, and independently verified the argument before including it in the manuscript.

ChatGPT and Claude were also used for language editing, improving exposition, and assisting with \LaTeX{} preparation. All mathematical claims, proofs, citations, and conclusions appearing in the final manuscript have been reviewed and validated by the authors, who take full responsibility for the content and correctness of the paper.
\end{aidisclosure}

\section{Preliminaries}\label{sec:prelim}
A \emph{graph} $G=(V,E)$ is finite, simple and undirected: $V$ is a finite set and
$E\subseteq\bigl\{\{u,v\}: u,v\in V,\ u\neq v\bigr\}$. We write $N_G(v) \coloneqq \bigl\{u:\{u,v\}\in E\bigr\}$ and
$\deg_G(v) \coloneqq |N_G(v)|$, omitting the subscript when $G$ is clear from the context. For $d>0$, a
\emph{$d$-embedding} of $G$ is a map $\gamma:V\to\R^d$; the pair $(G,\gamma)$ is a
\emph{$d$-embedded graph}. A \emph{Boolean-featured graph} is a $d$-embedded graph with
$\gamma(V)\subseteq\{0,1\}^d$. Throughout, the separating graphs use only one Boolean node feature, denoted by
$P$.  Thus the separation already holds in the minimal setting $d=1$.

Fix a set $\Sigma$ of functions $\sigma:\R\to\R$, called \emph{activation functions}. The language
$\Sigma\MPLang$ over $d$-embeddings is generated by the grammar
\[
e\ \Coloneqq\ 1\ \mid\ P_i\ \mid\ ae\ \mid\ e+e\ \mid\ \dm e\ \mid\ \sigma(e),
\qquad i\in[d],\ \sigma\in\Sigma,\ a\in\R.
\]
The semantics on a $d$-embedded graph $(G,\gamma)$ assigns to each expression $e$ and each node
$v$ a real number $e(G,\gamma)(v)$, as follows:
\[
\begin{aligned}
1(G,\gamma)(v)& \coloneqq 1, &
P_i(G,\gamma)(v)& \coloneqq \gamma(v)_i,\\
(ae)(G,\gamma)(v)& \coloneqq a\cdot e(G,\gamma)(v), &
(e_1+e_2)(G,\gamma)(v)& \coloneqq e_1(G,\gamma)(v)+e_2(G,\gamma)(v),\\
(\dm e)(G,\gamma)(v)& \coloneqq \textstyle\sum_{u\in N(v)}e(G,\gamma)(u), &
\sigma(e)(G,\gamma)(v)& \coloneqq \sigma\bigl(e(G,\gamma)(v)\bigr).
\end{aligned}
\]
We freely use the abbreviations $e_1-e_2 \coloneqq e_1+(-1)e_2$ and $m \coloneqq m\cdot 1$ for integer
constants, and we note that the identity activation is implicit in the grammar: linear
terms need not pass through any $\sigma$. 
Throughout we use
\[
\relu(x) \coloneqq \max\{0,x\},\qquad
\trelu(x) \coloneqq \min\{1,\max\{0,x\}\},\qquad
\bool(x) \coloneqq \begin{cases}1,&x>0,\\ 0,&x\le 0,\end{cases}
\]
for the ReLU, truncated ReLU, and unit-step (Booleanisation) activations, respectively.
We also use $\sgn(x)=+1$ if $x>0$, $\sgn(x)=0$ if $x=0$, and
$\sgn(x)=-1$ if $x<0$. 
For a single activation function $\sigma$, we write $\sigma\MPLang$ for
$\{\sigma\}\MPLang$; in particular $\mathrm{ReLU}\MPLang:=\{\relu\}\MPLang$
and $\mathrm{TrReLU}\MPLang:=\{\trelu\}\MPLang$.

\begin{definition}[Booleanisation of a query]\label{def:boolq}
For an expression $e$, its \emph{Booleanisation} $e_\B$ is the Boolean query
\[
e_\B(G,\gamma)(v) \coloneqq \begin{cases}1,&\text{if } e(G,\gamma)(v)>0,\\ 0,&\text{otherwise.}\end{cases}
\]
Expressions $e,e'$ are \emph{numerically equivalent} on a class $\mathcal C$ of embedded
graphs if $e(G,\gamma)=e'(G,\gamma)$ for all $(G,\gamma)\in\mathcal C$, and \emph{Boolean
equivalent} on $\mathcal C$ if $e_\B(G,\gamma)=e'_\B(G,\gamma)$ for all
$(G,\gamma)\in\mathcal C$. For a language $L$ we write $\bool(L) \coloneqq \{e_\B: e\in L\}$ for its
class of Boolean queries.
\end{definition}
\begin{definition}[Eventually constant functions]\label{def:evconst}
A function $\sigma:\R\to\R$ is \emph{eventually constant} if there exist reals
$x^\sigma_-\le x^\sigma_+$ and constants $C^\sigma_-,C^\sigma_+\in\R$ such that
$\sigma(x)=C^\sigma_-$ for all $x\le x^\sigma_-$ and $\sigma(x)=C^\sigma_+$ for all
$x\ge x^\sigma_+$.
\end{definition}
For instance, $\trelu$, $\bool$, $\sgn$, the hard sigmoid and the hard tanh are eventually
constant, whereas $\relu$ and $\mathrm{id}$ are not. 
\section{The main theorem}
We first recall that $\mathrm{TrReLU}\MPLang$ is subsumed by $\mathrm{ReLU}\MPLang$ \citep{Barcelo+2026}.
\begin{proposition}\label{prop:containment}
Every $\mathrm{TrReLU}\MPLang$ expression is numerically equivalent, on all embedded graphs,
to a $\mathrm{ReLU}\MPLang$ expression. Consequently,
\[
\bool(\mathrm{TrReLU}\MPLang)\subseteq\bool(\mathrm{ReLU}\MPLang).
\]
\end{proposition}
\begin{proof}
For all $x\in\R$ we have $\trelu(x)=\relu(x)-\relu(x-1)$. Define a translation $T$ that
recursively replaces every subexpression of the form $\trelu(e)$ by
$\relu(T(e))-\relu(T(e)-1)$ and commutes with all other constructors. A straightforward
structural induction shows that $T(e)$ is numerically equivalent to $e$ on every embedded
graph, and numerical equivalence implies Boolean equivalence.
\end{proof}

Our main result is, as follows.
\begin{theorem}[Main theorem]\label{thm:main}
Let $\Sigma$ be any collection of eventually constant activation functions. There exists a
$\mathrm{ReLU}\MPLang$ expression $\Estar$ such that no $\Sigma\MPLang$ expression (with
arbitrary real coefficients) is Boolean equivalent to it on the class of finite graphs with a
single Boolean feature.
\end{theorem}
As a corollary, we resolve the  problem left open in our previous work \cite{Barcelo+2026}:
\begin{corollary}\label{cor:trrelu}
On finite Boolean-featured graphs,
\[
\bool\bigl(\mathrm{TrReLU}\MPLang\bigr)\ \subsetneq\ \bool\bigl(\mathrm{ReLU}\MPLang\bigr),
\]
and this holds already on graphs with a single Boolean feature.
\end{corollary}
\begin{proof}
The containment $\bool(\mathrm{TrReLU}\MPLang)\subseteq\bool(\mathrm{ReLU}\MPLang)$ is
Proposition~\ref{prop:containment}; strictness follows from Theorem~\ref{thm:main} with
$\Sigma=\{\trelu\}$, applied to the Boolean query
$(\Estar)_\B\in\bool(\mathrm{ReLU}\MPLang)$.
\end{proof}
\begin{remark}[GNN formulation]\label{rem:gnn}
By the correspondence between $\Sigma\MPLang$ and $(\Sigma\cup\{\mathrm{id}\})$-GNNs
\citep{Geerts_2022}, and since $\mathrm{ReLU}\MPLang$ captures exactly ReLU-GNNs,
Corollary~\ref{cor:trrelu} states that $\{\mathrm{TrReLU},\mathrm{id}\}$-GNNs are strictly
weaker than ReLU-GNNs with respect to Boolean queries on Boolean-featured graphs.
\end{remark}
\paragraph{Proof strategy.}
Define the $\mathrm{ReLU}\MPLang$ expression (over one Boolean feature $P$)
\begin{equation}\label{eq:Estar}
\boxed{\ \Estar\  \coloneqq \ \dm\relu\bigl(2\dm P-\dm 1+1\bigr)\ -\ \dm 1\ +\ 4\ }
\end{equation}
Unravelling the semantics, its Booleanisation tests, at a node $v$,
\begin{equation}\label{eq:Estar-explicit}
\begin{aligned}
(\Estar)_\B(G)(v)=1\iff{}
&\sum_{w\in N(v)}\max\Bigl\{0,\ 2\,\#\{x\in N(w):P(x)=1\}-\deg(w)+1\Bigr\}\\
&{}-\deg(v)+4>0.
\end{aligned}
\end{equation}
We evaluate expressions on a three-parameter family of graphs $G_N(u_1,u_2,\alpha)$, whose
nodes fall into classes that are indistinguishable to $\Sigma\MPLang$ expressions. On one
distinguished node class, denoted $O$, the expression $\Estar$ takes the value
$2N(u_1+u_2-\alpha)$ and thus detects on which side of the plane $u_1+u_2=\alpha$ the
parameters lie. By contrast, we show that every expression built from eventually constant
activations becomes \emph{asymptotically local}: for all sufficiently large admissible $N$,
its value on $O$ depends only on $N$ and $\alpha$, not on $u_1$ or $u_2$. Choosing two
nearby rational parameter points with the same $\alpha$ but lying on opposite sides of the
plane $u_1+u_2=\alpha$ yields two graphs on which any candidate expression takes the same value at $O$,
whereas $\Estar$ takes values of opposite sign.
\begin{figure}[t]
\centering
\begin{tikzpicture}[xscale=1.5,yscale=1.15]
\foreach \j in {0,...,4}{
  \node[draw,circle,inner sep=1.6pt,fill=white] (a\j) at (\j,1.4) {};
  \node[above=1pt of a\j,font=\scriptsize,inner sep=1pt] {$a_{\j}$};
  \node[draw,circle,inner sep=1.6pt,fill=white] (b\j) at (\j,0) {};
  \node[below=1pt of b\j,font=\scriptsize,inner sep=1pt] {$b_{\j}$};
}
\foreach \j in {0,...,4}{ \draw (a\j) -- (b\j); }
\foreach \j/\k in {0/1,1/2,2/3,3/4}{ \draw  (a\j) -- (b\k); }
\draw (a4) .. controls (4.6,0.7) and (0.6,1.9) .. (b0);
\end{tikzpicture}
\caption{The regular gadget of Lemma~\ref{lem:gadgets} for $n=5$, $d=2$.}\label{fig:gadget}
\end{figure}
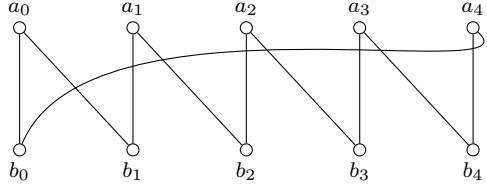
\section{The graph family}\label{sec:graphs}
We now construct the graph family. Let
$\Theta \coloneqq \bigl(\tfrac14,\tfrac13\bigr)^2\times\bigl(\tfrac12,\tfrac23\bigr)$ and consider
$\theta=(u_1,u_2,\alpha)\in\Theta$. For rational $\theta\in\Theta$,
we say an integer $N\ge1$ is
\emph{$\theta$-admissible} if $Nu_1$, $Nu_2$ and $2N\alpha$ are integers; the
$\theta$-admissible $N$ are exactly the positive multiples of the least common denominator
of $u_1$, $u_2$ and $2\alpha$. Put
\begin{equation}\label{eq:params}
r_i=N(1+u_i),\qquad b_i=N(1-u_i),\qquad t=2N\alpha .
\end{equation}
Then
\begin{equation}\label{eq:bounds}
0<b_i<N<r_i<2N,\qquad N<t<2N,\qquad
r_i+b_i=2N,\qquad r_i-b_i=2Nu_i.
\end{equation}
\begin{lemma}[Regular gadgets]\label{lem:gadgets}
Let $A,B$ be disjoint $n$-element sets.
\begin{enumerate}
\item For every $0\le d\le n$, there is a simple $d$-regular bipartite graph between $A$ and $B$.
\item If $n$ is even, then for every $0\le d<n$ there is a simple $d$-regular graph on $A$.
\end{enumerate}
\end{lemma}
\begin{proof}
Identify each set with $\mathbb Z_n$. For (1), join $j\in A$ to
$j,j+1,\ldots,j+d-1\in B$ (indices mod $n$); each vertex on either side
gets exactly $d$ distinct neighbours since $d\le n$. (See
Figure~\ref{fig:gadget}.) For (2), if $d$ is
even, join each $j\in A$ to $j\pm1,\ldots,j\pm d/2$ (mod $n$); this is a
simple $d$-regular graph since $d/2 < n/2$, so no two of these $d$
neighbours coincide and $j$ is not joined to itself. If $d$ is odd, use
the same construction for $d-1$ and add the antipodal matching
$\{j, j+n/2\}$ (mod $n$), which requires $n$ even to be a well-defined
fixed-point-free involution, and is disjoint from the edges already
added since $n/2$ is not among $\pm1,\dots,\pm(d-1)/2$ once $d<n$. 
\end{proof}
Fix a rational $\theta\in\Theta$ and a $\theta$-admissible $N$. Take five disjoint classes,
each of size $2N$,
\[
O,\quad W_1^+,W_1^-,\quad W_2^+,W_2^-.
\]
Using Lemma~\ref{lem:gadgets}, place a $t$-regular graph inside $O$, an $r_i$-regular graph
inside each $W_i^\pm$, a $b_i$-regular bipartite gadget between $W_i^+$ and $W_i^-$, and a
perfect matching between $O$ and each $W_i^\pm$. Add no other edges, and set the Boolean node
feature to
\[
P=1\text{ on }W_1^+\cup W_2^+,
\qquad P=0\text{ elsewhere}.
\]
Call the resulting Boolean-featured graph $G_N(u_1,u_2,\alpha)$, or
$G_N(\theta)$ for short; it is depicted in Figure~\ref{fig:quotient}.
Thus $G_N(\theta)$ includes not only the underlying graph, but also the feature
map $\gamma:V\to\{0,1\}$ given by $\gamma(v)=P(v)$. Accordingly, for a $\Sigma\MPLang$ expression $h$ and a node $v$, we write
$h(G_N(\theta))(v)$ for the value of $h$ at $v$, rather than
$h(G_N(\theta),\gamma)(v)$.
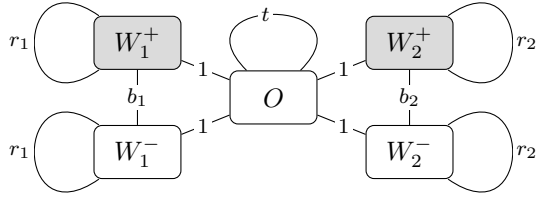
\begin{figure}[t]
\centering
\begin{tikzpicture}[xscale=0.6,yscale=0.6,
  cls/.style={draw, rounded corners=3pt, minimum width=1.15cm, minimum height=0.7cm,
              fill=white, font=\small, inner sep=2pt},
  fea/.style={cls, fill=black!14},
  lab/.style={font=\scriptsize, fill=white, inner sep=1.2pt}]
\node[cls] (O) at (0,0) {$O$};
\node[fea] (W1p) at (-3,1.2) {$W_1^+$};
\node[cls] (W1m) at (-3,-1.2) {$W_1^-$};
\node[fea] (W2p) at (3,1.2) {$W_2^+$};
\node[cls] (W2m) at (3,-1.2) {$W_2^-$};
\draw (O) -- node[lab,pos=.55] {$1$} (W1p);
\draw (O) -- node[lab,pos=.55] {$1$} (W1m);
\draw (O) -- node[lab,pos=.55] {$1$} (W2p);
\draw (O) -- node[lab,pos=.55] {$1$} (W2m);
\draw (W1p) -- node[lab] {$b_1$} (W1m);
\draw (W2p) -- node[lab] {$b_2$} (W2m);
\draw (O) to[out=45,in=135,looseness=5] node[lab,left] {$t$} (O);
\draw (W1p) to[out=145,in=215,looseness=5] node[lab,left] {$r_1$} (W1p);
\draw (W1m) to[out=145,in=215,looseness=5] node[lab,left] {$r_1$} (W1m);
\draw (W2p) to[out=35,in=-35,looseness=5] node[lab,right] {$r_2$} (W2p);
\draw (W2m) to[out=35,in=-35,looseness=5] node[lab,right] {$r_2$} (W2m);
\end{tikzpicture}
\caption{A depiction of the graph $G_N(u_1,u_2,\alpha)$. A loop labelled $d$ denotes a
$d$-regular graph inside the class; an edge labelled $d$ denotes a $d$-regular bipartite
gadget. Shaded classes carry $P=1$. Recall that $r_i=N(1+u_i)$, $b_i=N(1-u_i)$, and
$t=2N\alpha$.}\label{fig:quotient}
\end{figure}
\section{Evaluation of the separating expression on the graph family}\label{sec:evaluation}
We now evaluate the separating expression $\Estar$ on the graphs $G_N(\theta)$. The key
observation is that any two nodes in the same class are interchangeable as far as
$\Sigma\MPLang$ expressions are concerned.
\begin{lemma}\label{lem:constancy}
For every $\Sigma\MPLang$ expression $h$, the map $h(G_N(\theta))(\cdot)$ is constant on
each of the five classes.
\end{lemma}
\begin{proof}
Recall that a partition $\mathcal{C}$ of the vertex set of a graph is \emph{equitable} if,
for every two classes $X,Y\in\mathcal{C}$, there is a number $d_{XY}$ such that every vertex
in $X$ has exactly $d_{XY}$ neighbours in $Y$.
Note that the coarsest equitable partition of a graph
corresponds to the partition resulting from running colour refinement on the graph \citep{grohe2014dimension}.
The partition
$\mathcal{C}=\{O,W_1^+,W_1^-,W_2^+,W_2^-\}$ is equitable: with the classes ordered as
$O,\ W_1^+,\ W_1^-,\ W_2^+,\ W_2^-$, the inter-class neighbour counts $d_{XY}$ are given by
the quotient matrix
\[
D=(d_{XY})_{X,Y\in\mathcal C}
=
\begin{pmatrix}
t & 1 & 1 & 1 & 1\\
1 & r_1 & b_1 & 0 & 0\\
1 & b_1 & r_1 & 0 & 0\\
1 & 0 & 0 & r_2 & b_2\\
1 & 0 & 0 & b_2 & r_2
\end{pmatrix}.
\]
Indeed, every vertex in $O$ has $t$ neighbours in $O$ and one neighbour in each of the four
classes $W_i^\pm$, while every vertex in $W_i^+$ has one neighbour in $O$, $r_i$ neighbours
in $W_i^+$, and $b_i$ neighbours in $W_i^-$; the description for $W_i^-$ is symmetric. In
particular,
\begin{equation}\label{eq:degrees}
\deg(v)=t+4\ \text{ for }v\in O,
\qquad
\deg(v)=r_i+b_i+1=2N+1\ \text{ for }v\in W_i^\pm.
\end{equation}
We prove the claim by structural induction on $h$.
\emph{Base cases.}
The expression $1$ has value $1$ at every vertex, and $P$ is constant on each class by
construction: it has value $1$ on $W_1^+$ and $W_2^+$, and value $0$ on the other three
classes.
\emph{Linear and activation cases.}
If $h_1$ and $h_2$ are constant on every class, then so are $ah_1$, $h_1+h_2$, and
$\sigma(h_1)$, since these operations are evaluated pointwise.
\emph{Aggregation case.}
Suppose that $h_1$ is constant on every class, with value $(h_1)_Y$ on class $Y$. For any
class $X\in\mathcal C$ and any vertex $v\in X$, equitability gives
\[
\begin{aligned}
(\dm h_1)(G_N(\theta))(v)
  &=\sum_{u\in N(v)}h_1(u)
  =\sum_{Y\in\mathcal C}\sum_{u\in N(v)\cap Y}h_1(G_N(\theta))(u)
  =\sum_{Y\in\mathcal C}d_{XY}(h_1)_Y,
\end{aligned}
\]
which depends only on the class $X$ and not on the particular choice of $v\in X$. Hence
$\dm h_1$ is constant on every class, completing the induction.
\end{proof}
In view of Lemma~\ref{lem:constancy}, we write $h_X$ for the value of $h$ on class $X$, and
$h_X(N,\theta)$ when the instance needs to be displayed.
Recall the separating expression
\[
\Estar\  \coloneqq \ \dm\relu\bigl(2\dm P-\dm 1+1\bigr)\ -\ \dm 1\ +\ 4,
\]
and let $z \coloneqq 2\dm P-\dm 1+1$. From the construction,
\[
z_O=1-t<0,
\qquad z_{W_i^+}=r_i-b_i=2Nu_i,
\qquad z_{W_i^-}=b_i-r_i=-2Nu_i.
\]
Since every $o\in O$ has $t$ neighbours in $O$ and one in each class $W_i^\pm$,
\[
(\dm\relu(z))_O=2N(u_1+u_2).
\]
As $(\dm 1)_O=t+4$ and $t=2N\alpha$, we obtain the following \emph{key property}.
\begin{proposition}\label{prop:Estar-value}
For every rational $\theta\in\Theta$, every $\theta$-admissible $N$, and every $o\in O$,
\[
\Estar(G_N(\theta))(o)=2N(u_1+u_2-\alpha).
\]
Thus $(\Estar)_\B(G_N(\theta))(o)=1$ exactly when $u_1+u_2>\alpha$.
\end{proposition}
\section{Asymptotic locality}\label{sec:lemma}
Call a polynomial in $\R[u_1,u_2,\alpha]$ \emph{cylindrical} if it lies in $\R[\alpha]$,
$\R[u_1,\alpha]$, or $\R[u_2,\alpha]$. In other words, such polynomials never involve both $u_1$ and
$u_2$. 
\begin{definition}[Sign decompositions]\label{def:signs}
A finite set $\mathcal P$ of nonzero cylindrical polynomials
determines, for each sign assignment $s:\mathcal P\to\{+1,-1\}$, the
open (possibly empty) \emph{cell}
\[
\Theta_s \coloneqq \{\theta\in\Theta:\ \sgn(p(\theta))=s(p)\
\text{for all }p\in\mathcal P\},
\]
where $p\in\R[\alpha]$ or $\R[u_i,\alpha]$ is evaluated at the
corresponding coordinates of $\theta$. The cells are pairwise disjoint
and cover $\Theta\setminus Z(\mathcal P)$, where
$Z(\mathcal P)$ is the zero set $\bigcup_{p\in\mathcal P}\{\theta\in\Theta:
p(\theta)=0\}$.
\end{definition}
Put
\[
A_0=\R+N\R[N,\alpha],\qquad A_i=\R+N\R[N,u_i,\alpha]\quad(i=1,2).
\]
\begin{definition}[Local vectors]\label{def:local}
A \emph{class-vector} is a tuple
$q=(q_O,q_{W_1^+},q_{W_1^-},q_{W_2^+},q_{W_2^-})$ of polynomials in
$N,u_1,u_2,\alpha$. Writing $w_i(q) \coloneqq q_{W_i^+}+q_{W_i^-}$, we call $q$
\emph{local} when
\[
q_O,\,w_1(q),\,w_2(q)\in A_0
\qquad\text{and}\qquad
q_{W_i^\pm}\in A_i\quad(i=1,2).
\]
\end{definition}
Here ``local'' refers to parameter dependence. The form of $A_0$ and $A_i$ also
ensures that every nonconstant parameter-dependent term carries a
positive power of $N$.

\begin{definition}[Asymptotic locality]\label{def:asylocal}
A $\Sigma\MPLang$ expression $h$ is \emph{asymptotically local} (on the
family $G_N(\theta)$) if there exist
\begin{enumerate}
\item a finite set $\mathcal P_h$ of nonzero cylindrical polynomials,
and
\item for each sign assignment $s$ on $\mathcal P_h$ with
$\Theta_s\neq\emptyset$, a local class-vector $\Phi^s$,
\end{enumerate}
such that for every rational $\theta\in\Theta_s$ there is a threshold
$N_0(h,\theta)\in\N$ with
\[
h_X(N,\theta)=\Phi^s_X(N,\theta)
\qquad\text{for every class $X$ and every $\theta$-admissible }
N\ge N_0(h,\theta),
\]
where evaluation substitutes $N$ and the coordinates of $\theta$
occurring in $\Phi^s_X$.
\end{definition}
We will show that every $\Sigma\MPLang$ expression is asymptotically
local (Section~\ref{sec:locality-holds}) when $\Sigma$ contains only eventually constant activation functions. Before doing so, we show that the main theorem
already follows.
\section{Proof of Theorem~\ref{thm:main}}
The only ingredient needed is the specialisation of
Definition~\ref{def:asylocal} to the class $O$: if $h$ is
asymptotically local then, on each cell $\Theta_s$, there is a
\emph{single} polynomial
$F_s \coloneqq \Phi^s_O\in A_0\subseteq\R[N,\alpha]$, mentioning neither $u_1$
nor $u_2$, with
\begin{equation}\label{eq:O-blind}
h\bigl(G_N(\theta)\bigr)(o)=F_s(N,\alpha)
\text{ for all } o\in O,\ \text{all rational }\theta\in\Theta_s,\
\text{all admissible } N\ge N_0(h,\theta).
\end{equation}
In particular, let $\theta,\theta'\in\Theta_s$ be two rational points with the
same $\alpha$-coordinate.  For every common admissible $N$ with
$N\ge\max\{N_0(h,\theta),N_0(h,\theta')\}$, the expression $h$ takes the same
value on the class $O$ in $G_N(\theta)$ and $G_N(\theta')$ -- no matter how their
parameters $u_1,u_2$ differ. The target query, by contrast,
flips exactly across the surface $u_1+u_2=\alpha$
(Proposition~\ref{prop:Estar-value}). The proof below drives a wedge
between the two, using one classical fact.
\begin{fact}\label{fact:poly-open}
A polynomial $p\in\R[x_1,\dots,x_n]$ that vanishes on a nonempty open
subset $U\subseteq\R^n$ is the zero polynomial.
\end{fact}
\begin{proof}[Proof of Theorem~\ref{thm:main}]
Suppose some $\Sigma\MPLang$ expression $f$ is Boolean equivalent to
$\Estar$ on all graphs with one Boolean feature. We know that $f$ is asymptotically local; let
$\mathcal P$ be its finite set of nonzero cylindrical polynomials and
$F_s$ the cell-wise polynomials of \eqref{eq:O-blind}.

\smallskip
\noindent
We first select a base point on the critical surface $\alpha=u_1+u_2$.  Consider the diagonal
piece $H=\{(x,y,x+y):(x,y)\in(\tfrac14,\tfrac13)^2\}\subseteq\Theta$ of
the surface $u_1+u_2=\alpha$.
We first show that there is a point of $H$ at which every $p\in\mathcal P$
is nonzero. For each $p\in\mathcal P$, we define $\tilde p\in\R[u_1,u_2]$ by
$\tilde p(u_1,u_2) \coloneqq p(u_i,u_1+u_2)$ if $p\in\R[u_i,\alpha]$,
and $\tilde p(u_1,u_2) \coloneqq p(u_1+u_2)$ if $p\in\R[\alpha]$.
Each polynomial $\tilde p$ is nonzero.  For instance, suppose
$p\in\R[u_1,\alpha]$ and $\tilde p$ were identically zero.  Then, for all
$u_1,\alpha$,
\[
  p(u_1,\alpha)
  =
  p(u_1,u_1+(\alpha-u_1))
  =
  \tilde p(u_1,\alpha-u_1)
  =
  0,
\]
contradicting that $p$ is nonzero.  The case $p\in\R[u_2,\alpha]$ is the same,
using $u_1=\alpha-u_2$, and the case $p\in\R[\alpha]$ is immediate.
This is where cylindricity is indispensable: a polynomial mentioning
both $u_1$ and $u_2$, such as $u_1+u_2-\alpha$ itself, could vanish
identically on $H$.
Their product $\prod_{p\in\mathcal P}\tilde p$ is thus also nonzero,
and by Fact~\ref{fact:poly-open}, it is nonzero somewhere on the open box
$(\tfrac14,\tfrac13)^2$, so by continuity and the density of $\Q^2$
it is nonzero at a rational point $(a,b)\in(\tfrac14,\tfrac13)^2$.
Thus every $p\in\mathcal P$ is nonzero at the point
$\theta^*\coloneqq(a,b,a+b)$, by construction.
Then $\theta^*\in\Theta_s$, where $s(p)\coloneqq\sgn(p(\theta^*))$.

\smallskip
\noindent
We next create two nearby points on opposite sides of the critical surface, yet inside one cell. Since
$\mathcal P$ is finite and its members are continuous and nonzero at
$\theta^*$, we may pick a rational $\varepsilon>0$ so small that
$a\pm\varepsilon\in(\tfrac14,\tfrac13)$ and every $p\in\mathcal P$
keeps its sign at
\[
\theta^\pm \coloneqq (a\pm\varepsilon,\,b,\,a+b).
\]
Then $\theta^-,\theta^+\in\Theta_s$, they share $\alpha_0 \coloneqq a+b$, and
$u_1+u_2-\alpha$ evaluates to $\mp\varepsilon$ at $\theta^\mp$: the two
points lie on opposite sides of the critical surface.

\smallskip
\noindent
We are now ready to obtain a contradiction. Choose $N$ to be a common multiple of the least
common denominators of $\theta^-$ and $\theta^+$ with
$N\ge\max\{N_0(f,\theta^-),N_0(f,\theta^+)\}$; such $N$ exist,
since there are arbitrarily large common multiples.
Both
graphs $G_N(\theta^\pm)$ are then defined, and by \eqref{eq:O-blind},
applied on the cell $\Theta_s$ containing both points, for every
$o\in O$
\[
f\bigl(G_N(\theta^-)\bigr)(o)=F_s(N,\alpha_0)
=f\bigl(G_N(\theta^+)\bigr)(o),
\]
one and the same real number. Yet
Proposition~\ref{prop:Estar-value} gives
$\Estar\bigl(G_N(\theta^\mp)\bigr)(o)=\mp2N\varepsilon$, of opposite
signs. So the Booleanisations of $f$ agree at the $O$-vertices of the
two graphs while those of $\Estar$ differ, contradicting Boolean
equivalence.
\end{proof}
\section{Eventually constant implies asymptotically local}\label{sec:locality-holds}
It remains to prove that every $\Sigma\MPLang$ expression is asymptotically
local whenever all activations in $\Sigma$ are eventually constant.

\begin{lemma}\label{lem:normalform}
Let $\Sigma$ be a collection of eventually constant activation functions. Then every $\Sigma\MPLang$ expression is asymptotically local.
\end{lemma}
\begin{proof}
We proceed by structural induction on  $\Sigma\MPLang$ expressions $h$.
Along the induction the sets $\mathcal P_h$
only grow ($\mathcal P_{h'}\subseteq\mathcal P_h$ for $h'$ a
subexpression of $h$); a sign assignment on the larger set restricts to
one on the smaller, whose cell contains the (nonempty) finer cell, so
inductive data remains available and valid on every finer cell, with
unchanged thresholds.

\smallskip
\noindent
\emph{Base cases.} For $h=1$ and $h=P$, take
$\mathcal P_h \coloneqq \emptyset$ (the unique empty sign assignment has cell
$\Theta$) and the constant vectors
\[
\Phi_{1} \coloneqq (1,1,1,1,1),\qquad
\Phi_{P} \coloneqq (0,1,0,1,0),
\]
recording that $P=1$ exactly on $W_1^+\cup W_2^+$. Constants lie in
$A_0\subseteq A_i$ and the pair-sums are constants, so both vectors are
local; the defining equalities hold exactly, with $N_0=1$.

\smallskip
\noindent
\emph{Linear cases.} For $h=a\,h_1$ and $h=h_1+h_2$, take
$\mathcal P_h \coloneqq \mathcal P_{h_1}$, resp.\
$\mathcal P_{h_1}\cup\mathcal P_{h_2}$, and the corresponding linear
combinations of the inductive vectors. The classes $A_0,A_i$ and the
pair-sum conditions are closed under linear combinations, and class
values combine linearly. So the resulting vector is local and continues to describe the class-values
for all sufficiently large admissible $N$
(by choosing the maximum of the inherited thresholds as the new threshold).

\smallskip
\noindent
\emph{Aggregation.} For $h=\dm h_1$, take
$\mathcal P_h \coloneqq \mathcal P_{h_1}$: aggregation makes no decisions. Let
$D$ be the quotient adjacency operator of Lemma~\ref{lem:constancy} and
$q \coloneqq \Phi^s$ the inductive vector; define $\Phi^s_{\text{new}} \coloneqq Dq$,
i.e., substitute $t=2N\alpha$, $r_i=N(1+u_i)$, $b_i=N(1-u_i)$ into the
quotient matrix.
For every rational $\theta$ in the cell and every admissible
$N\ge N_0(h_1,\theta)$, the inductive vector gives the class-values of
$h_1$.  Hence the class-values of $\dm h_1$ are obtained by applying the
quotient matrix $D$ to that vector.
For locality, compute from
the quotient matrix:
\[
(Dq)_O=2N\alpha\, q_O+w_1(q)+w_2(q),
\qquad
w_i(Dq)=2q_O+2N\,w_i(q),
\]
and individually
$(Dq)_{W_i^\pm}=q_O+N(1\pm u_i)\,q_{W_i^+}+N(1\mp u_i)\,q_{W_i^-}$.
By induction $q_O,w_1(q),w_2(q)\in A_0$, and multiplying an element of
$A_0$ by $2N\alpha$ or $2N$ lands in $N\R[N,\alpha]\subseteq A_0$; hence
$(Dq)_O$ and $w_i(Dq)$ lie in $A_0$.
Furthermore, $(Dq)_{W_i^\pm}\in A_i$,
as the multipliers $N(1\pm u_i)$ carry a
factor $N$ and involve only $u_i$, and $q_O\in A_0\subseteq A_i$,
$q_{W_i^\pm}\in A_i$ by induction.
So $Dq$ is local.

\smallskip
\noindent
\emph{Activation.} For $h=\sigma(h_1)$ with $\sigma\in\Sigma$
eventually constant, fix a sign assignment $s$ on $\mathcal P_{h_1}$
with nonempty cell and write each coordinate of the inductive vector as
\[
\Phi^s_X=c_X+\sum_{k\ge1}a_{X,k}N^k,
\qquad c_X\in\R,\quad a_{X,k}\ \text{cylindrical},
\]
which is precisely what membership in $A_0,A_i$ provides. Enlarge
$\mathcal P_h$ by all nonzero $a_{X,k}$ (over all $s$ and the five
classes: finitely many polynomials). On any cell of the enlarged set,
for each class $X$ either all $a_{X,k}$ are zero polynomials, and
then the argument of $\sigma$ at $X$ is the
\emph{parameter-free} constant $c_X$, so $\sigma(h_1)_X$ is
$\sigma(c_X)$, or the leading nonzero coefficient in $N$ has a fixed sign
on the cell, so the argument is a polynomial function of $N$ diverging
to $+\infty$ or $-\infty$, and $\sigma(h_1)_X$ is eventually
$C^\sigma_+$ or $C^\sigma_-$ (Definition~\ref{def:evconst}). In every
case the outcome is a real constant determined by the cell, reached
beyond a threshold depending on the rational parameter point (how large
$N$ must be for the leading term to dominate and for the argument to
leave $[x^\sigma_-,x^\sigma_+]$). The resulting constant vector is
local, all its entries and pair-sums being real numbers. Note that the
polynomials added to $\mathcal P_h$ are coefficients of entries of a
local vector, hence lie in $\R[\alpha]$ or $\R[u_i,\alpha]$ and
cylindricity is maintained.
\end{proof}
\bibliographystyle{plain}

\begin{thebibliography}{1}

\bibitem{Barcelo+2026}
Pablo Barcel{\'o}, Floris Geerts, Matthias Lanzinger, Klara Pakhomenko, and Jan Van~den Bussche.
\newblock A logical view of {GNN}-style computation and the role of activation functions.
\newblock {\em Proc. ACM Manag. Data}, 4(2), May 2026.

\bibitem{benedikt2025decidabilitygraphneuralnetworks}
Michael Benedikt, Chia-Hsuan Lu, and Tony Tan.
\newblock Decidability of graph neural networks via logical characterizations.
\newblock {\em ACM Trans. Comput. Logic}, 27(2), April 2026.

\bibitem{Geerts_2022}
Floris Geerts, Jasper Steegmans, and Jan Van~den Bussche.
\newblock On the expressive power of message-passing neural networks as global feature map transformers.
\newblock In {\em Foundations of Information and Knowledge Systems}, page 20–34. Springer, 2022.

\bibitem{grohe2014dimension}
Martin Grohe, Kristian Kersting, Martin Mladenov, and Erkal Selman.
\newblock Dimension reduction via colour refinement.
\newblock In {\em European Symposium on Algorithms}, pages 505--516. Springer, 2014.

\end{thebibliography}

\end{document}